\newif\ifarxiv
\arxivfalse

\documentclass[10pt,twocolumn,letterpaper]{article}

\usepackage[pagenumbers]{cvpr} 

\usepackage{xurl}   

\usepackage{xurl}

\usepackage{amsmath,amsfonts,amsthm,amssymb}
\usepackage{mathtools}
\usepackage{bm}
\usepackage{nicefrac}
\usepackage{microtype}
\usepackage{lipsum}

\usepackage{color,xcolor}
\usepackage{epsfig}
\usepackage{graphicx}

\usepackage{adjustbox}
\usepackage{array}
\usepackage{booktabs}
\usepackage{colortbl}
\usepackage{wrapfig}
\usepackage{hhline}
\usepackage{multirow}
\usepackage{subcaption}
\usepackage{caption}
\usepackage{float}

\usepackage{changepage}
\usepackage{extramarks}
\usepackage{fancyhdr}
\usepackage{lastpage}
\usepackage{setspace}
\usepackage{soul}
\usepackage{xspace}
\usepackage{diagbox}
\usepackage{url}

\usepackage{algorithm}
\usepackage{algpseudocode}
\usepackage{enumerate}
\usepackage{enumitem}  
\usepackage{makecell}
\usepackage{pifont}
\usepackage{titlecaps}
\ifarxiv\else
\usepackage[accsupp]{axessibility}
\fi
\usepackage{framed}

\usepackage{epigraph}

\usepackage{CJKutf8}

\usepackage{float}

\newcommand{\model}{ShadowDancer\xspace}
\newcommand{\shadow}{Shadow\xspace}
\newcommand{\shadowlibrary}{Shadow Library\xspace}

\renewcommand{\paragraph}[1]{\vspace{0.1cm}\noindent\textbf{#1}}

\definecolor{MyDarkBlue}{rgb}{0,0.08,1}
\definecolor{MyAqua}{rgb}{0,0.7,0.7}
\definecolor{MyDarkGreen}{rgb}{0.02,0.6,0.02}
\definecolor{MyDarkRed}{rgb}{0.8,0.02,0.02}
\definecolor{MyDarkOrange}{rgb}{0.40,0.2,0.02}
\definecolor{MyPurple}{RGB}{111,0,255}
\definecolor{MyRed}{rgb}{1.0,0.0,0.0}
\definecolor{MyGold}{rgb}{0.75,0.6,0.12}
\definecolor{MyDarkgray}{rgb}{0.66, 0.66, 0.66}
\definecolor{Cardinal}{rgb}{0.549,0.082,0.082}

\newif\ifdrafting
\draftingtrue 
\ifdrafting
    \newcommand{\jc}[1]{\textcolor{MyDarkGreen}{\begin{CJK}{UTF8}{gbsn}[Jin: #1]\end{CJK}}}
    \newcommand{\cl}[1]{\textcolor{MyDarkBlue}{\begin{CJK}{UTF8}{gbsn}[Claude: #1]\end{CJK}}}
    
\else
    \newcommand{\jc}[1]{}
    \newcommand{\cl}[1]{}
    
\fi

\definecolor{cvprblue}{rgb}{0.21,0.49,0.74}
\usepackage[pagebackref,breaklinks,colorlinks,allcolors=cvprblue]{hyperref}

\def\paperID{*****} 
\def\confName{CVPR}
\def\confYear{2026}

\title{\makebox[0pt][r]{\raisebox{\dimexpr .5\fontcharht\font`S-.5\height\relax}{\includegraphics[scale=0.031]{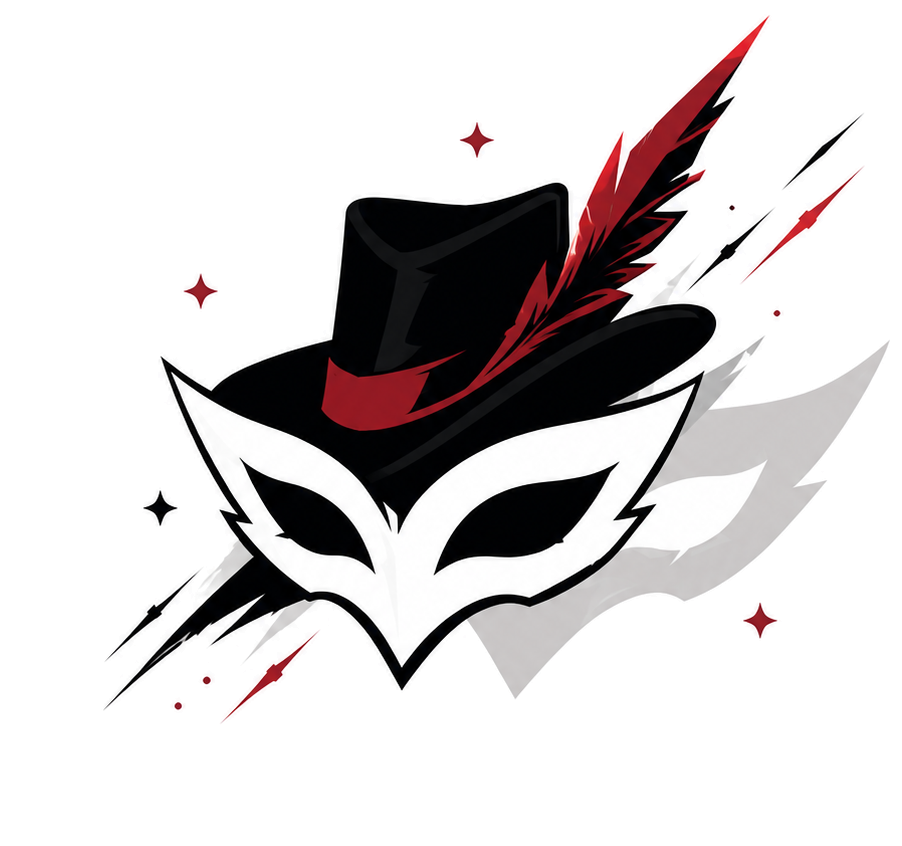}}\,}ShadowDancer: Teaching Video World Models Any Action by Learning\\ Unified Dynamics Representations from a Video and Its Shadow}

\author{
Jin Cao$^{1}$ \quad Zian Meng$^{1,2}$ \quad Kaipeng Zhang$^{1\dagger}$\\
$^{1}$Alaya Lab \qquad $^{2}$Shanghai Innovation Institute\\
{\url{https://ShadowDancer-1.github.io}}
}

\begin{document}
\twocolumn[
\maketitle
\vspace{-9mm}
\begin{center}
    \includegraphics[trim={0px 0px 0px 0px}, clip, width=\linewidth]{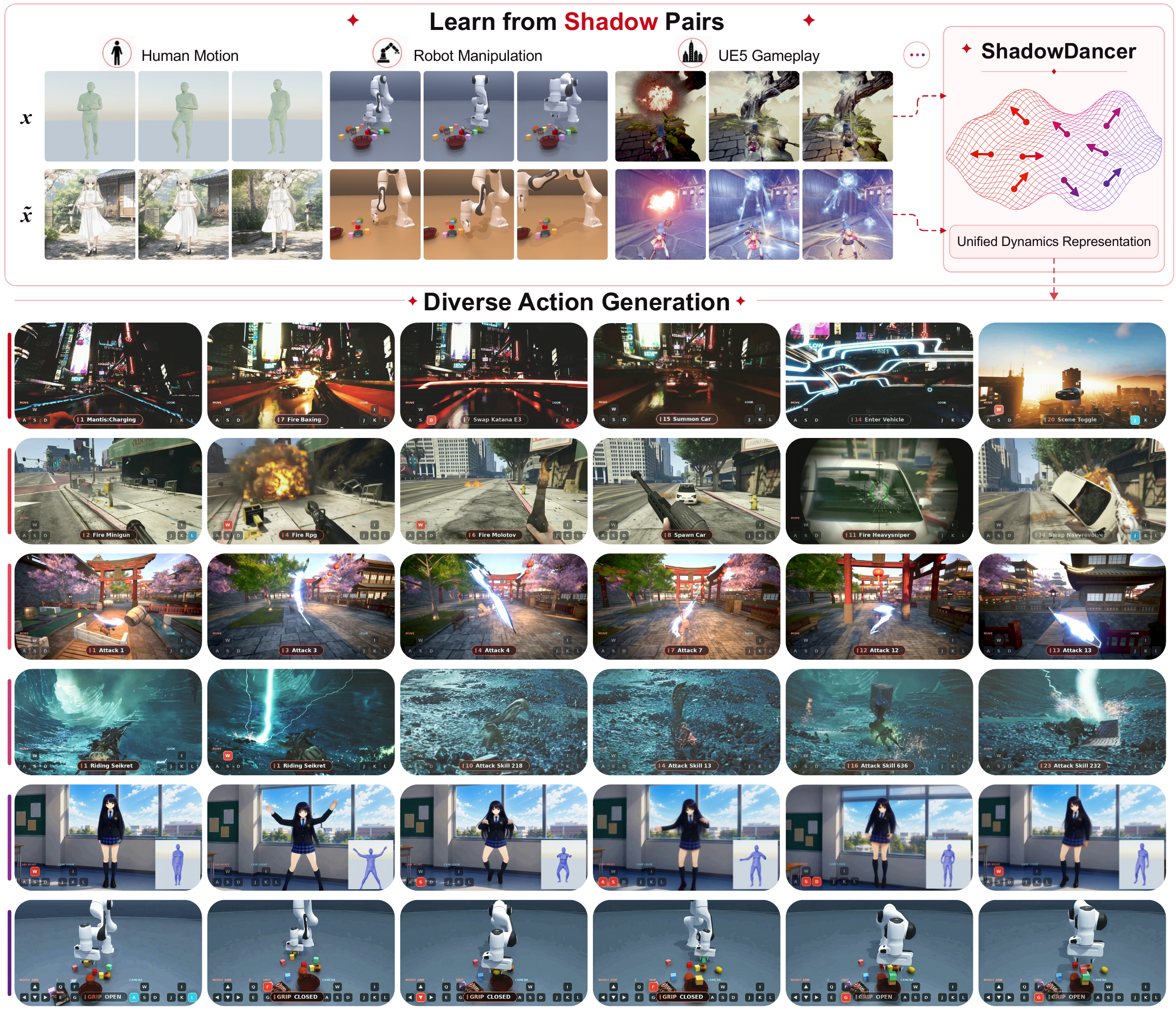}
\end{center}
\vspace{-0.4cm}
\captionof{figure}{\textbf{ShadowDancer learns any action from a video and its shadow.}
Top: shadow pairs $(x, \tilde{x})$ replay one dynamics under independently resampled appearance, across heterogeneous sources such as human motion, robot manipulation, and open-world gameplay, and distill it into a unified dynamics representation $z_{1:T-1}$.
Bottom: the same latent interface drives diverse commanded actions, spanning first- and third-person combat, driving, locomotion, and robot manipulation.}
\label{fig:teaser}

\bigbreak
]
{\renewcommand{\thefootnote}{}\footnotetext{$^{\dagger}$Corresponding author.}}
\begin{abstract}
We present \model, a novel approach to any-action, frame-level control of interactive video world models.
The obstacle is representational: existing interfaces either encode an action loosely, leaving how it unfolds for the model to improvise, or encode it exactly through structured signals that serve one family and are hard to acquire, so precise control across diverse dynamics remains impractical.
Demonstration videos are the natural remedy, specifying any dynamics frame by frame; yet a video shows its dynamics only through one particular appearance, a single \emph{shadow} of the underlying dynamics, so actions learned from demonstrations transfer poorly to new scenes.
\model addresses this with two key innovations: (1) \emph{shadow pairs}, video pairs that replay the same dynamics under independently resampled appearance, constructed at scale by our \shadowlibrary, so that a dynamics family becomes controllable exactly when such pairs can be constructed for it; and (2) \emph{cross-shadow prediction}, which learns actions by predicting one shadow from the other, so that whatever the pairing resamples is discarded by construction and whatever it preserves becomes the action, yielding a unified dynamics representation that drives a block-causal world model.
Any demonstrated clip thus becomes a reusable action asset, replayed in new environments without action labels, motion estimators, or fine-tuning.
Experiments demonstrate improved action transfer and long action rollout over strong latent-action and interactive world model baselines across diverse dynamics families, with an average blinded win rate of $86\%$ in rollout comparisons.
We show video results at \url{https://ShadowDancer-1.github.io}.

\end{abstract}
\section{Introduction}
\label{sec:intro}

Interactive world models have emerged as a transformative capability in video generation, synthesizing persistent virtual worlds that users can explore in real time~\cite{genie3_deepmind_2025,he2025matrix,sun2025worldplay,mao2025yume,oasis2024,team2026advancing,gao2026infinite}.
However, despite the boundless variety of dynamics a world can host, existing approaches lack a unified means of specifying how those dynamics should unfold.
The result is a stubborn trade-off between precision and generality: a behavior can be specified loosely for arbitrary dynamics, or exactly for one narrow family, but rarely both at once.
This leaves worlds that can be explored but not yet directed, restricting their applicability for interactive entertainment and world simulation.
We therefore target \emph{any-action, frame-level control}: precision, so that a command specifies how an action unfolds frame by frame, and generality, so that the same interface extends to any dynamics family.

The bottleneck is not generative capacity, as video diffusion backbones already synthesize complex articulated motion well~\cite{videoworldsimulators2024,wan2025wan,chen2025skyreels,kong2024hunyuanvideo}; it lies in the \emph{representation} of actions: how the user communicates which dynamics should unfold.
Any-action, frame-level control is at its core a representation learning problem, and each existing interface can be read as one candidate representation.
Symbolic commands~\cite{he2025matrix,yu2025gamefactory,oasis2024} were designed around logged interactions such as keystrokes: they name an event from a finite vocabulary and leave its trajectory for the generator to improvise, so holding a key longer repeats an action yet hardly reshapes it.
Free-form text~\cite{zhu2026incantation,mao2025yume} removes the vocabulary but not the ambiguity, since language describes dynamics rather than supplying it.
Structured motion inputs such as 3D human motion~\cite{cao2025uni3c,fang2026threedimo}, point tracks~\cite{lee2026generative,shin2026motionstream}, and camera pose trajectories do reach frame-level precision, but each format serves a single family, demands its own conditioning machinery, and depends on specialized yet often brittle estimators, so the precise signal it promises is hard to obtain in practice.
What remains is the demonstration itself, a natural specification that is temporally dense and format-free, and the way humans are taught new actions.
Latent action models already exploit it, distilling the transitions of a reference clip into continuous latent actions $z$ that condition generation~\cite{edwards2019imitating,rybkin2018learning,bruce2024genie,chen2024igor,ye2024latent}.
But critically, such models are trained to reconstruct the very clip they read, and self-reconstruction fundamentally conflicts with \emph{control}: reconstruction asks $z$ merely to explain the observed appearance, whereas control asks it to transfer the underlying dynamics to new appearances~\cite{yang2025learning,garrido2026learninglatentactionworld,locatello2019challenging,khemakhem2020variational,jiang2026olaf}.
Trained without ever observing the same dynamics under two different appearances, the latent cannot tell which traits belong to the action and which merely co-occur with it: it entangles what moves with how it looks indiscriminately, and shrinking it~\cite{gaoadaworld} or regularizing it toward semantic features~\cite{jiang2026olaf} operates within the same supervision: invariance can only be encouraged as a penalty, and the unchosen residue can resurface as spurious motion in the generation.
What the task calls for is to learn the representation from data in which dynamics is separated from appearance by construction.

\looseness=-1 To address this challenge, we propose \textbf{\model}, built on the idea of observing the same dynamics \emph{twice}.
Our approach introduces two key innovations.
(1)~The \emph{\shadowlibrary}: we realize observing-twice as \emph{\shadow pairs}, where a \shadow of a video is a second render of the same dynamics under a different appearance, frame-synchronized in motion while subject, scene, materials, and lighting are independently resampled.
Constructing shadows is a protocol rather than an engine: replay the dynamics and resample everything else.
The \shadowlibrary implements this protocol across animation suites, open-world games, and robotic simulators.
(2)~\emph{Cross-shadow prediction}: an encoder extracts latent actions from one video, and a decoder predicts the other shadow from those latents under the shadow's own appearance context.
Since the target appearance is already supplied, the prediction requires transferring dynamics rather than reconstructing appearance, making invariance a property of the supervision rather than a penalty.
The pairing defines the action itself: whatever the pair shares becomes the controllable factor.
What the representation couples to is thus no longer an accident of the data but a choice of the protocol.
Pairing on a body motion while resampling the camera yields camera-free control, whereas pairing on a camera trajectory across scenes yields pure camera control.
More generally, any dynamics family becomes controllable once a shadow pair can be constructed for it, with no vocabulary to design and no labels to collect.
The resulting latent stream is a unified dynamics representation, driving every action family through one interface.
It conditions a pretrained video diffusion backbone~\cite{chen2025skyreels,wan2025wan}, which we convert into a block-causal autoregressive generator~\cite{gu2025long,huang2025self} for interactive rollout.
Each demonstrated clip thereby becomes a \emph{variable-length action asset}, the reusable unit of control in \model.

In use, directing an action takes one demonstration: the dynamics of the clip is extracted in a single frozen-encoder pass, stored as an asset, composed with other assets along the condition stream, and replayed in a new environment.
Unlike symbolic interfaces, which specify what should happen and leave how it happens to the generator, \model supplies a dense control signal read from the trajectory itself, so the timing, amplitude, and style of an action are inherited from the demonstration.
We evaluate this interface in two settings: action transfer, which tests how precisely a demonstration from one environment is re-enacted in another, and long action rollout, which tests how well stored assets compose over time.
In summary, our contributions are threefold:
\begin{itemize}
    \item \looseness=-1 We introduce the \shadow learning paradigm, which learns, from a video and its \shadow, dynamics representations invariant to everything the pairing resamples, resolves latent-action non-identifiability by construction, and provides a formal analysis of when the shared dynamics is identifiable.
    \item \looseness=-1 We build the \shadowlibrary, a source-agnostic pairing protocol implemented through paired-rendering scripts across animation suites, open-world games, and robotic simulators, supplying shadow supervision for families including human, camera, and object dynamics without a single action label.
    \item We present \model, an any-action interactive world model driven by dense, variable-length action assets, and show improved action transfer and long action rollout over strong latent-action and interactive world model baselines across the evaluated dynamics families.
\end{itemize}

\begin{figure*}[t]
    \centering
    \includegraphics[width=\linewidth]{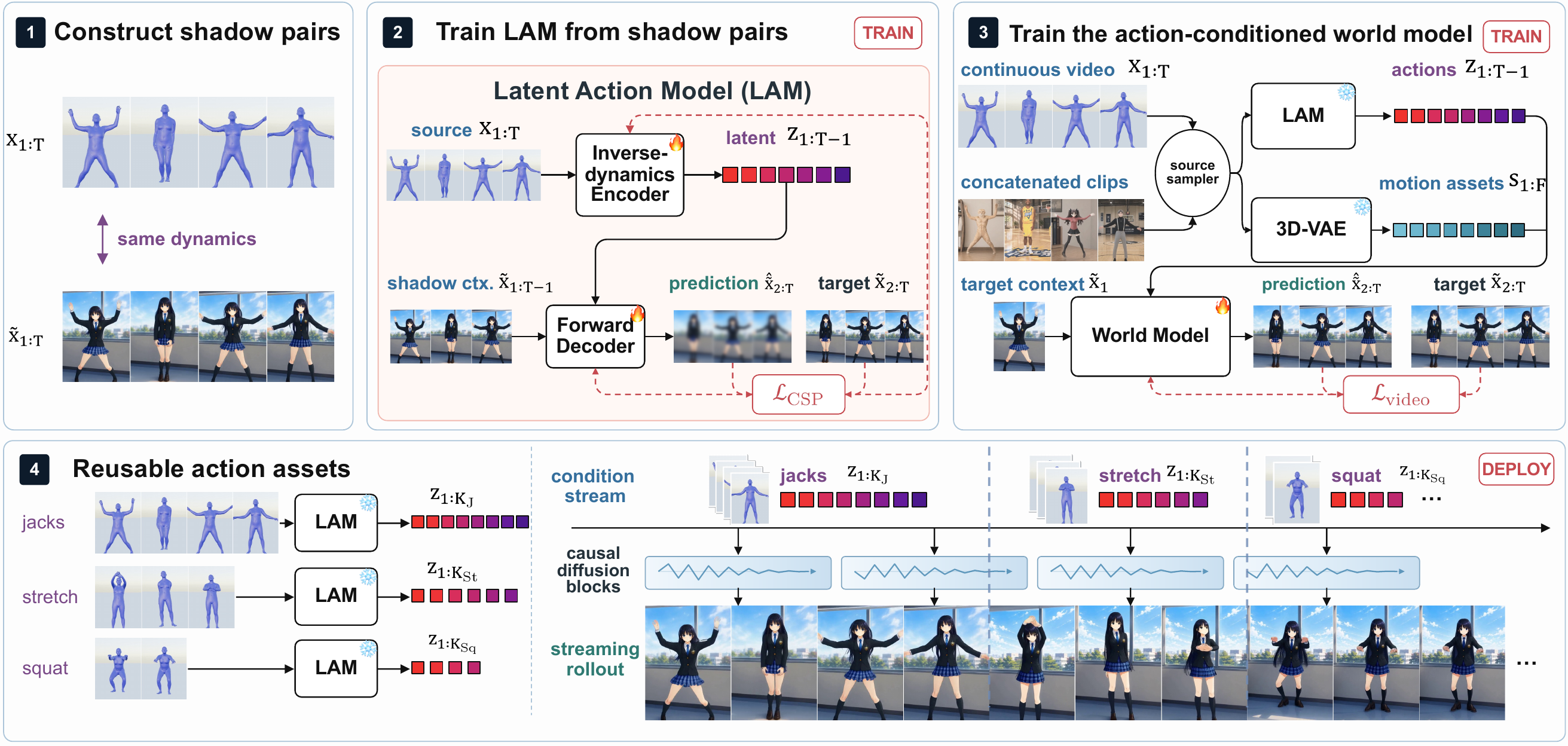}
    \caption{\textbf{Overview of \model}, taking human motion as the running example; the pipeline is identical for camera, object, and every other dynamics family.
    (1)~A shadow pair renders one dynamics $d$ twice, $x{=}R(d,c)$ and $\tilde{x}{=}R(d,\tilde{c})$, with the remaining factors independently resampled.
    (2)~Cross-shadow prediction trains the LAM: the encoder reads each $z_t$ from the source transition, and the decoder predicts the next shadow frame from $(\tilde{x}_t, z_t)$, so only what the pair shares survives in $z$.
    (3)~With the LAM and 3D-VAE frozen, a video diffusion backbone is fine-tuned into a block-causal world model, conditioned on the action trajectory $z$ and on source motion assets $s$, and trained on mixed continuous and concatenated sources.
    (4)~At deployment, variable-length action assets $(z, s)$ are stored, composed along the condition stream, and streamed through causal blocks; a new action costs one pass through the frozen encoders.}
    \label{fig:method_overview}
\end{figure*}

\section{Related Work}
\label{sec:related}

\paragraph{Interactive Video World Models.}
\looseness=-1 World models predict future observations and support planning or interactive simulation in games, robotics, and driving~\cite{ha2018world,hafner2023mastering,agarwal2025cosmos,gao2024vista}, and recent systems generate explorable worlds at impressive fidelity and horizon~\cite{genie3_deepmind_2025,oasis2024,he2025matrix,sun2025worldplay,mao2025yume,team2026advancing,gao2026infinite}.
Their control interfaces, however, are symbolic or textual: frame-level keyboard and mouse states logged from game engines~\cite{oasis2024,alonso2024diffusion,yu2025gamefactory,xiao2025worldmem}, or free-form text prompts~\cite{zhu2026incantation,mao2025yume}.
These commands specify what happens but not how (Sec.~\ref{sec:intro}), and telemetry-logged interfaces further bind each model to the action schema of its data-collection pipeline~\cite{he2025matrix,tang2025hunyuan,hong2025relic,ye2026mind}; structured motion inputs~\cite{cao2025uni3c,fang2026threedimo,lee2026generative,shin2026motionstream} are precise instead, but per-family and hard to acquire accurately.
\model instead conditions on a dense per-frame latent extracted from a reference clip, so that the speed, amplitude, and style of an action are supplied by the reference, and any dynamics family with constructible shadow pairs is controlled through one interface.

\paragraph{Latent Actions as Control Interfaces.}
\looseness=-1 Latent action models (LAMs) infer controls directly from unlabeled video, serving as interfaces for interactive generation~\cite{bruce2024genie,gaoadaworld,jang2025dreamgen}, cross-embodiment policy learning~\cite{chen2024igor,ye2024latent,bu2025univla,kim2025uniskill,chen2025villa,chen2025moto}, and world-model pretraining~\cite{edwards2019imitating,rybkin2018learning,wang2025co,garrido2026learninglatentactionworld}.
Unlike imitation learning, which aims to execute a demonstrated behavior on a target embodiment, our focus is the representation itself: a reusable dynamics signal, extracted from a demonstration, that transfers across appearances and environments.
The core difficulty is that inverse-dynamics latents couple control to whatever appearance happens to co-occur with it: clip-local reconstruction admits shortcut solutions through context cues~\cite{yang2025learning,garrido2026learninglatentactionworld} and is non-identifiable across contexts~\cite{locatello2019challenging,khemakhem2020variational}, so the same underlying dynamics need not map to a consistent latent~\cite{jiang2026olaf}.
Existing remedies constrain the latent (information bottlenecks, quantization)~\cite{bruce2024genie,gaoadaworld} or regularize it toward motion-centric references~\cite{jiang2026olaf}; in practice, regularization reduces but does not eliminate the entanglement.
All of these operate on a single observation of each transition, so the desired invariance is imposed as a penalty rather than exhibited by the data.
Shadow pairing changes the data instead: each dynamics is observed twice under independently resampled appearance, so invariance holds by construction and the objective merely has to read it off.

\paragraph{Invariance by Construction.}
Our supervision transplants a proven recipe from self-supervised learning: a representation is defined by the pairs it is trained on, since what a pair varies is discarded and what it preserves is kept.
Contrastive learning pairs two augmentations of one image, so the representation drops crops and color while retaining visual identity~\cite{chen2020simple}; CLIP pairs an image with its caption, so it retains the semantics that language can express~\cite{radford2021learning}; masked and predictive video objectives extend the same logic to spatiotemporal structure~\cite{tong2022videomae,assran2025v}.
In each case the objective does not create the invariance; it reads off the invariance that the pairing already wrote into the data.
Dynamics, however, resists pixel-space augmentation: a post-hoc transform of one video can hardly resample its appearance while preserving its motion frame by frame.
The shadow pair moves the augmentation into the renderer, replaying the trajectory and resampling everything else, which grounds cross-shadow prediction (Sec.~\ref{sec:method_lam}) as a sound way to recover the preserved dynamics, invariant to everything resampled.

\section{Method}
\label{sec:method}

\paragraph{Overview.}
Fig.~\ref{fig:method_overview} summarizes \model.
The method has three parts, in pipeline order: shadow pairs make dynamics identifiable in the data (Sec.~\ref{sec:method_formulation}); \emph{cross-shadow prediction} distills it into the unified representation $z$ (Sec.~\ref{sec:method_lam}); and a video diffusion backbone, conditioned on $z$ for control and on source assets for motion detail, becomes a block-causal interactive world model whose deployment we describe last (Sec.~\ref{sec:method_worldmodel}).
The \shadowlibrary (Sec.~\ref{sec:method_library}) implements the pairing protocol at scale.
Three objects recur throughout: the LAM extracts a latent trajectory $z$ invariant to everything its pairing resamples (\emph{appearance}, Sec.~\ref{sec:method_formulation}); a frozen 3D-VAE encodes the aligned source-detail stream $s$; and together $a = (z, s)$ is the reusable \emph{action asset} the world model consumes, while the \emph{unified dynamics representation} always denotes $z$ alone, unified in the sense of one encoder, one latent space, and one interface for every dynamics family.

\subsection{The Shadow Formulation}
\label{sec:method_formulation}

We model a video $x = (x_1, \dots, x_T)$ as the output of a rendering process
\begin{equation}
    x = R(d, c),
    \label{eq:render}
\end{equation}
where $d = (d_1, \dots, d_T)$ is a time-varying \emph{dynamics} trajectory, the factor that drives change between frames, and $c$ collects the remaining generative factors, such as subject identity, scene, materials, and lighting, which we call \emph{appearance}.
In this notation the metaphor of our title is literal: a video is a \emph{shadow} of its dynamics, one rendering of $d$ under one particular $c$.
Observing $x$ alone, $d$ and $c$ are entangled: many different $(d, c)$ pairs can explain the same pixels, which is the root cause of the shortcut and non-identifiability failures of latent action learning~\cite{yang2025learning,jiang2026olaf,locatello2019challenging}.

A \textbf{\shadow pair} comprises two independent renders of the same dynamics,
\begin{equation}
    \big(\,x = R(d, c),\; \tilde{x} = R(d, \tilde{c})\,\big),\;\; \tilde{c} \sim p(c),
    \label{eq:shadowpair}
\end{equation}
where the two videos are frame-synchronized: frame $t$ of $x$ and frame $t$ of $\tilde{x}$ share the identical $d_t$.
Both videos are shadows of the same $d$, and each is therefore the \shadow of the other: the relation is symmetric, and ``the \shadow of a video'' refers to a fellow projection of its dynamics rather than a projection of the video itself.
Because $\tilde{c}$ is resampled independently of $c$, the only factor deliberately preserved across the two rendering processes is $d$; source appearance cannot systematically predict target appearance.

\looseness=-1 The split between $d$ and $c$ in Eq.~\eqref{eq:render} is not fixed a priori; it is chosen by the pairing protocol.
Replaying a body motion while resampling the camera assigns the camera to $c$; replaying a camera trajectory while resampling the scene assigns it to $d$.
This operational definition is what makes the paradigm any-action: designating a factor as controllable requires only constructing pairs that preserve it.
Shadow pairing therefore differs from data augmentation: augmentation encourages invariance to nuisances chosen within a fixed task, whereas the pairing chooses the task variable itself, by deciding what is preserved.
\textbf{Accordingly, \emph{appearance} in this paper denotes the factors a pairing resamples, and \emph{dynamics} what it preserves.}
Dynamics need not exclude visual traits: a fire spell without its flames is no longer the same action, so the protocol keeps such traits in the representation by simply preserving them across the pair, and they transfer together with the motion.

\subsection{Cross-Shadow Prediction}
\label{sec:method_lam}

\paragraph{Objective.}
Cross-shadow prediction turns the pairing constraint into a learning objective.
We instantiate it with a latent action model (LAM): an inverse-dynamics encoder models each transition $(x_t, x_{t+1})$ with a variational posterior $q_\phi(z_t \mid x_t, x_{t+1})$ over latent actions $z_t \in \mathbb{R}^{d_z}$, $t = 1, \dots, T\!-\!1$, and a forward decoder $p_\theta$ performs next-frame prediction.
Standard LAMs~\cite{gaoadaworld,jiang2026olaf} read and predict the same clip.
Our decisive difference is the cross-shadow arrangement: the encoder reads the source transition, while the decoder predicts the other shadow, reconstructing $\tilde{x}_{t+1}$ from the target's own previous frame $\tilde{x}_t$ and the source-derived action $z_t$.
Training minimizes the $\beta$-VAE objective~\cite{higgins2017betavae,alemi2017deep}:
\begin{align}
\mathcal{L}^{\mathrm{CSP}}_{\theta,\phi}
=
\frac{1}{T\!-\!1}\sum_{t=1}^{T-1}\Big(
&-\mathbb{E}_{q_\phi(z_t\mid x_t, x_{t+1})}\big[\log p_\theta(\tilde{x}_{t+1}\mid \tilde{x}_t, z_t)\big]
\notag\\
&\quad+\beta\,\mathrm{KL}\!\left(q_\phi(z_t\mid x_t, x_{t+1})\,\|\,p(z_t)\right)
\Big),
\label{eq:lam}
\end{align}
where $p(z_t)$ is a fixed prior $\mathcal{N}(0, I)$.

\paragraph{Why this works.}
\looseness=-1 Source-context information in $z_t$ beyond what is needed to infer $d$ cannot improve population prediction of $\tilde{x}$, whose appearance is supplied by its own context $\tilde{x}_t$ and independently resampled from $c$: given the pairing, appearance ceases to help the prediction, and the bottleneck, pressured by the KL term, has room only for what does help, the shared dynamics.
Unlike feature-alignment regularizers~\cite{jiang2026olaf}, the predictive supervision itself supplies this cross-context constraint.
With probability $p$ we also let a video pair with itself ($\tilde{x} = x$), recovering the standard single-video LAM objective as a degenerate case; these self-pairs admit unpaired real video into the same training stream (Sec.~\ref{sec:method_library}) as an empirical realism augmentation, outside the guarantee below.

\paragraph{Formal guarantee.}
This preference for dynamics is not only an intuition; we prove it.
Write $D$ for the shared dynamics, $X$ for the source video, $(B, Y)$ for the target-side context and prediction target, and $K_d$ for the law of $Y$ given $B$ under dynamics $d$.
Whenever the pairing protocol delivers
\begin{equation}
\underbrace{D = h(X)}_{\text{observable}},\quad
\underbrace{X \perp\!\!\!\perp (B,Y)\mid D}_{\text{independent resampling}},\quad
\underbrace{K_d {=} K_{d'} \Rightarrow d {=} d'}_{\text{effect separation}},
\label{eq:conditions}
\end{equation}
the minimal representation sufficient for predicting one shadow from the other is the shared dynamics itself, up to an invertible reparameterization (Theorem~\ref{thm:shadow_identification}); formal statements and proofs are in Supplementary Sec.~\ref{app:shadow_guarantee}.

\paragraph{Factor-selective control.}
Dynamics itself is composite: the observer's own motion (ego, \ie the camera) and the motion of everything observed (scene) are superimposed in every video.
Pair construction separates them within one latent space and specifies the readout: pairs preserving only the camera trajectory supervise a \texttt{cam} head, pairs preserving only the scene dynamics supervise a \texttt{dyn} head, and pairs preserving both supervise a \texttt{full} head; what the \texttt{dyn} factor contains is itself defined by each family's pairing protocol, the arm trajectory alone in our robot manipulation pairs, the body motion in our human pairs.
At inference the same reference video can therefore be read as a pure camera action, a pure scene action, or their joint; the mechanism is detailed in Supplementary Sec.~\ref{app:factor_selective}.

\subsection{From Representation to World Model}
\label{sec:method_worldmodel}

\looseness=-1 The conditioning design follows two principles: (i) the action latent $z$ carries the unified, transferable control, and (ii) the source stream $s$ supplies the high-frequency motion detail that an invariance-constrained bottleneck is deliberately too small to carry.

\paragraph{Action-conditioned generation.}
We fine-tune a pretrained video diffusion transformer (DiT)~\cite{peebles2023scalable,chen2025skyreels,wan2025wan} on \shadowlibrary clips with flow matching~\cite{liu2023flow}: the frozen LAM's actions are fused into the timestep embedding for AdaLN modulation and attended through a dedicated cross-attention arm, while the source video is encoded by the 3D-VAE and injected as \emph{motion assets}, through channel concatenation at the input convolution and as additional cross-attention context.
Because source and target form a shadow pair, copying source appearance is not rewarded, so the generator is trained to read the source for motion detail rather than appearance; the cross-attention arm is zero-initialized, with details in Supplementary Sec.~\ref{app:conditioning}.

\paragraph{Block-causal conversion for interactive rollout.}
\looseness=-1 An interactive world model must generate forward-only over long horizons, while diffusion backbones are bidirectional and fixed-length.
Following recent causal conversions of video diffusion~\cite{gu2025long,huang2025self}, we fine-tune the bidirectional model into a block-causal generator: latent frames are grouped into blocks, attention is causal across blocks, and denoising is conditioned on the clean history of previous blocks.
At inference the model rolls out block by block with a key--value cache, consuming a stream of latent actions and emitting a stream of frames.

\paragraph{Actions as reusable assets.}
\looseness=-1 At inference, the model receives a first frame $\mathbf{I}_0$ and a demonstration, given as reference clips or stored action assets: the frozen encoder reads the demonstration into per-frame latent actions $z$, and the block-causal generator rolls out a video in which the depicted world re-enacts the demonstrated dynamics.
Because control is a latent trajectory rather than a per-frame key state, an ``action'' in \model is a \emph{variable-length dynamics segment}, stored as an action asset: extract $a = (z_{1:K}, s)$ from a clip of any duration, store it, replay it in any world, and compose segments by concatenation along the rollout, so teaching the model a new action takes a single pass through the frozen encoders, with no vocabulary design, labels, or fine-tuning.
Blocks group latent frames only for causal denoising, while the action stream stays per-frame, so a single action may span many blocks and its temporal detail is not tied to the block rate.
At deployment, discrete commands retrieve and concatenate short stored assets rather than supplying one continuous demonstration; we therefore expose the generator to the same discontinuous conditioning during training, replacing the continuous source, with some probability, by a concatenation of \emph{canonical action chunks} drawn from a fixed asset library, so a command at inference is a library lookup that matches the input statistics the generator was trained on.

\subsection{The \shadowlibrary}
\label{sec:method_library}

The shadow protocol, which replays the dynamics and resamples everything else, admits many implementations, and we deliberately build a library rather than a single engine.
Animation suites replay articulated human motion across characters, scenes, lighting, and cameras; open-world game environments replay trajectories across districts, weather, and time of day; robotic simulators replay manipulation across arms, objects, and tabletops; camera-trajectory sources replay the same path through different scenes.
All sources emit the same artifact, frame-synchronized clip sets that share one dynamics, so a new action family integrates by contributing one more shadow script.

Real-world video, which lacks exact shadows, enters the very same stream as degenerate self-pairs ($\tilde{x} = x$): it contributes visual diversity and realism priors, while the invariance pressure is supplied by the synthetic pairs.
Genuine shadow pairs supply the identifying signal; self-pairs broaden the visual support.

\section{Experiments}
\label{sec:exp}

\subsection{Implementation Details}
\label{sec:exp_impl}
Following~\cite{gaoadaworld,jiang2026olaf}, the LAM has latent dimension $d_z{=}32$ and is trained on the \shadowlibrary at half the world model's spatial resolution ($H/2{\times}W/2$) with $\beta{=}0.01$; real videos enter as degenerate self-pairs.
The world model is built on the SkyReels-V2-1.3B I2V DiT backbone~\cite{chen2025skyreels}, fine-tuned with flow matching and converted to a block-causal generator.
All experiments are conducted on NVIDIA H200 GPUs.
Additional implementation details are provided in Supplementary Sec.~\ref{app:exp_details}.

\paragraph{Datasets}
The \shadowlibrary spans articulated human motion (SMPL-X sequences re-rendered in Blender across characters, scenes, lighting, and cameras, together with paired renders driven by~\cite{cao2025uni3c} that replay one motion under a resampled character and environment), robotic manipulation (ManiSkill~\cite{tao2024maniskill3}), first-person open-world games (GTA- and Cyberpunk-style urban scenes with first-person weapon fire and driving), third-person games (Unreal Engine scenes; Monster Hunter-style action), and camera trajectories, both around articulated subjects and through static scenes (DL3DV~\cite{ling2023dl3dv}); unpaired real video (OpenX robot corpora~\cite{oneill2024openx}, Internet clips~\cite{ju2024miradata}) joins as self-pairs.
Held-out shadow pairs from every source form the evaluation sets; the full composition of the \shadowlibrary is tabulated in Supplementary Sec.~\ref{app:exp_details}.

\begin{table*}[t]
\centering
\caption{\textbf{Quantitative comparison of action transfer across dynamics families.} Best in \textbf{bold}.}
\label{tab:transfer}
\resizebox{\linewidth}{!}{
\begin{tabular}{l|cc|cc|cc|cc|cc}
\toprule
& \multicolumn{2}{c|}{Human motion} & \multicolumn{2}{c|}{First-person combat} & \multicolumn{2}{c|}{Third-person action} & \multicolumn{2}{c|}{Camera control} & \multicolumn{2}{c}{Robot manipulation} \\
Method & PSNR$\uparrow$ & LPIPS$\downarrow$ & PSNR$\uparrow$ & LPIPS$\downarrow$ & PSNR$\uparrow$ & LPIPS$\downarrow$ & ATE$\downarrow$ & RPE$\downarrow$ & PSNR$\uparrow$ & LPIPS$\downarrow$ \\
\midrule
Olaf-World~\cite{jiang2026olaf} & 18.2 & 0.288 & 13.0 & 0.532 & 12.6 & 0.506 & 0.072 & 0.021 & 14.0 & 0.478 \\
\model (ours) & \textbf{22.4} & \textbf{0.184} & \textbf{17.0} & \textbf{0.354} & \textbf{17.4} & \textbf{0.309} & \textbf{0.005} & \textbf{0.003} & \textbf{22.6} & \textbf{0.116} \\
\bottomrule
\end{tabular}}
\end{table*}

\begin{figure*}[t]
    \centering
    \includegraphics[width=\linewidth]{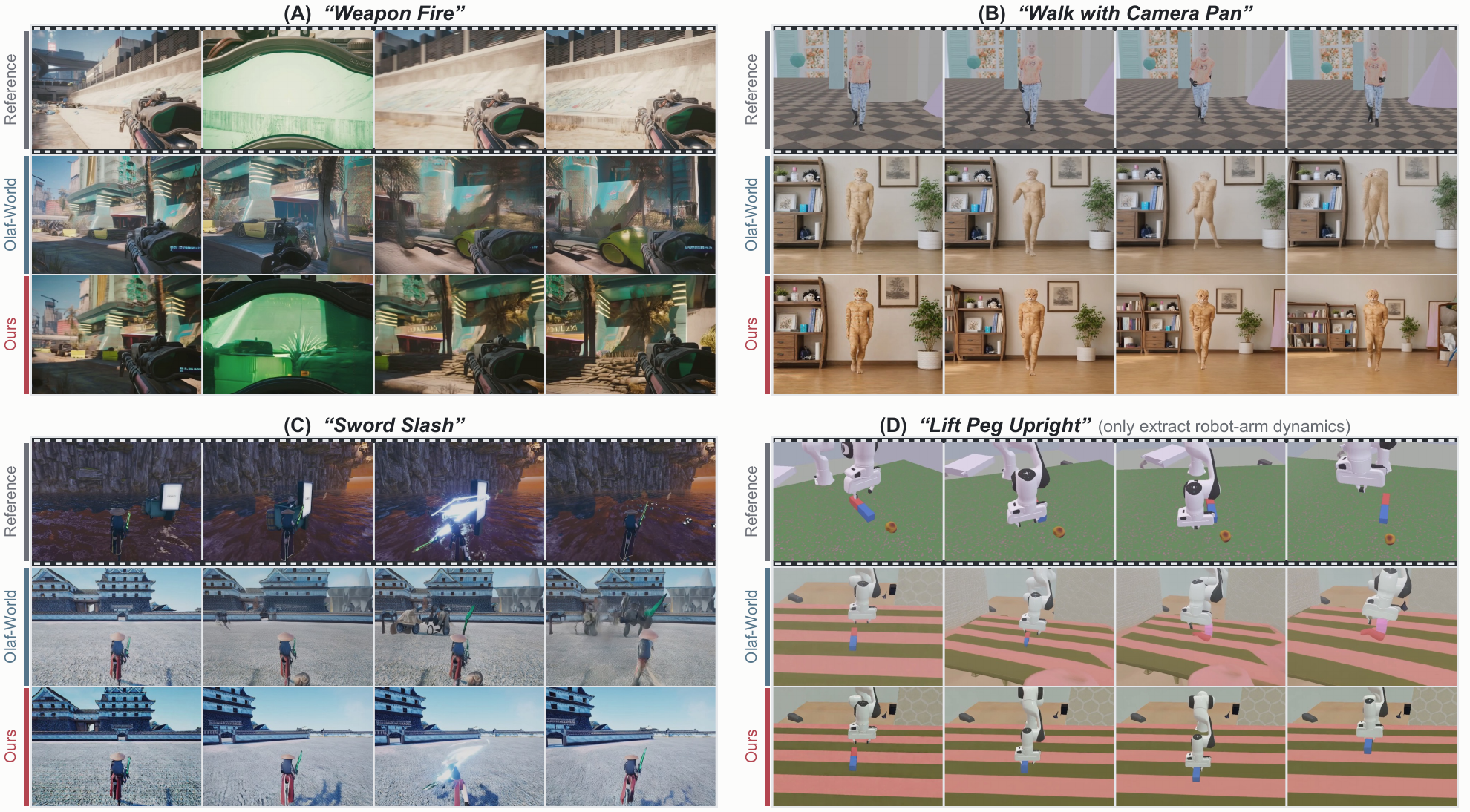}
    \caption{\textbf{Qualitative action transfer against Olaf-World.} Top: reference videos providing the action across four families; below: frame-synchronized generations in the new environment. Olaf-World warps subjects and leaks spurious motion; \model~re-enacts the reference dynamics faithfully.}
    \label{fig:olaf_vs_ours}
\end{figure*}

\subsection{Action Transfer Across Dynamics Families}
\label{sec:exp_transfer}
Our first experiment tests the central claim in one protocol: a single model with one latent interface re-enacts every family of dynamics in a new environment.
For every family in the \shadowlibrary (human motion, first-person combat, third-person action, robot manipulation, and camera control), we hold out shadow pairs, extract the action from video $x$, and generate from the first frame of its \shadow $\tilde{x}$, which realizes the identical, frame-synchronized dynamics.
Four families are scored by frame-aligned reconstruction against $\tilde{x}$ (PSNR, LPIPS~\cite{zhang2018unreasonable}); camera control is scored by trajectory error (ATE/RPE~\cite{sturm2012benchmark}) between the target trajectory and poses recovered from the generations with VGGT~\cite{wang2025vggt}.
The baseline is Olaf-World~\cite{jiang2026olaf}, the state-of-the-art self-reconstruction latent-action model: we match the training data and world-model backbone, inject its latent through the same routes, and retain its original single-head, $z$-only design, which has no source-asset stream.
Table~\ref{tab:transfer} reports the results: \model~leads on every family by a wide margin.
The gap is visible in Fig.~\ref{fig:olaf_vs_ours}: Olaf-World routinely produces warped, twisting subjects and spurious camera motion, because its entangled latent cannot cleanly extract the intended dynamics from the source, so what it transfers is the action mixed with whatever appearance happened to accompany it; with the shadow-trained latent, the same interface re-enacts the demonstrated dynamics faithfully in the new environment.
Latent-level probes corroborate this before any generation is involved: linear probes on $z$ show that cross-shadow training preserves action content while reducing leakage of the resampled scene (Supplementary Sec.~\ref{app:exp_details}).
The gain is not bought by blur: \model~roughly halves the Fr\'echet Video Distance of Olaf-World on three reconstruction families (Supplementary Sec.~\ref{app:exp_details}).

\begin{table}[t]
\centering
\caption{\textbf{Long action rollout.} Blinded 2AFC favor rate of \model~over each baseline; ${>}50\%$ favors ours.}
\label{tab:rollout}
\setlength{\tabcolsep}{4pt}
\renewcommand{\arraystretch}{1.15}
\resizebox{\linewidth}{!}{
\begin{tabular}{lccc}
\toprule
\model~(ours) over & Action Control & Action Fidelity & Long-horizon Consist. \\
\midrule
Olaf-World~\cite{jiang2026olaf} & $94\%$ & $95\%$ & $94\%$ \\
Yume-1.5~\cite{mao2025yume} & $88\%$ & $86\%$ & $91\%$ \\
LingBot-World~2.0~\cite{gao2026infinite} & $78\%$ & $83\%$ & $64\%$ \\
\bottomrule
\end{tabular}}
\end{table}
\begin{figure*}[t]
    \centering
    \includegraphics[width=\linewidth]{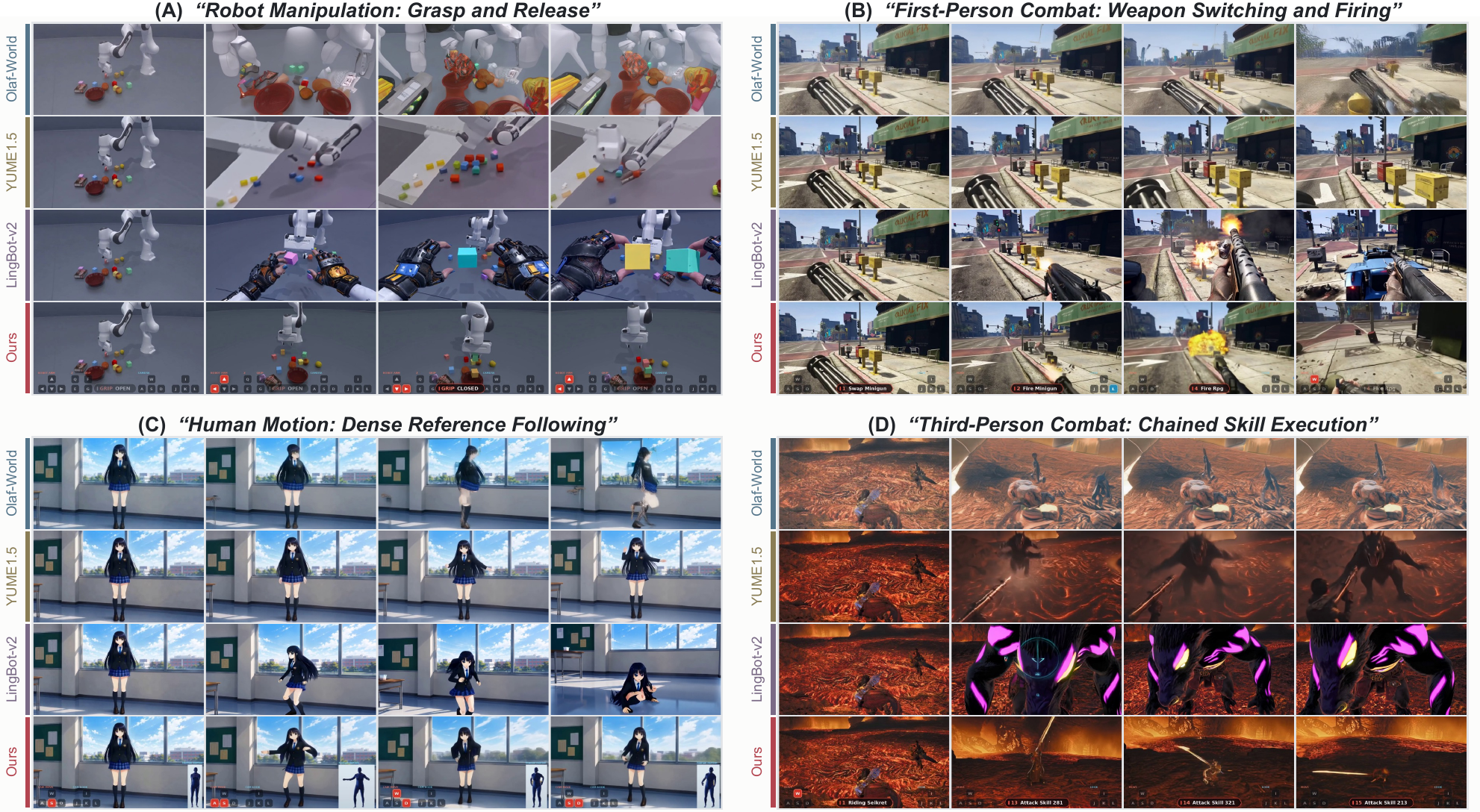}
    \caption{\textbf{Qualitative long action rollouts.} Each comparison starts from a shared first frame and follows the same command stream; rows are methods, columns are time steps along the rollout. Where the baselines drift or degrade, \model~stays aligned.}
    \label{fig:rollout_qual}
\end{figure*}

\subsection{Long Action Rollout}
\label{sec:exp_control}
Our second experiment evaluates whether learned action assets remain composable over long horizons.
Discrete commands are mapped to canonical action assets by lookup and streamed into long rollouts that chain several commands per video, spanning navigation, aiming and firing, and skill casts across first- and third-person worlds, together with human-motion and robot-manipulation streams.
Unlike the transfer setting, a chained command stream has no ground-truth realization, so control must be \emph{judged} rather than measured: following the preference protocol of~\cite{zhu2026incantation}, each \model~rollout is compared head-to-head against a baseline's in a blinded two-alternative forced choice (2AFC), judged by a VLM (Fable~5) over all pairwise comparisons.
The comparison runs along three action-centric axes: \emph{action control} (is the commanded action performed at all), \emph{action fidelity} (is the specific weapon/object/motion preserved), and \emph{long-horizon consistency} (no drift or degradation as the rollout extends); visual fidelity is excluded as orthogonal to control.
We compare against the state-of-the-art open interactive world models LingBot-World~2.0~\cite{gao2026infinite}, driven by a camera-pose trajectory with text-triggered events, and Yume-1.5~\cite{mao2025yume}, driven by text and keystrokes, alongside Olaf-World's latent interface on our backbone; each receives the same command stream through its native interface.
The comparison is deliberately system-level: the interfaces receive different information by design, and how faithfully a user's commands survive each interface is precisely what is being measured.
Table~\ref{tab:rollout} reports the favor rates, and Fig.~\ref{fig:rollout_qual} shows representative rollouts; where the baselines drift, hallucinate, or degrade, \model~stays aligned.

\begin{table}[t]
\centering
\caption{\textbf{Impact of key components.} From the $z$-only latent (Olaf-World), adding \emph{pairing}, then \emph{assets} ($=$\model). Row (a$'$) adds the assets to an \emph{unpaired} $z$, isolating the effect of pairing.}
\label{tab:ablation}
\resizebox{\linewidth}{!}{
\begin{tabular}{lccc|cc}
\toprule
Setting & $z$ & Paired $z$ & Assets & PSNR$\uparrow$ & LPIPS$\downarrow$ \\
\midrule
\textit{(a) Olaf-World ($z$ only)}      & \checkmark & $\times$   & $\times$   & 12.44 & 0.555 \\
\textit{(b) $+$ pairing}                & \checkmark & \checkmark & $\times$   & 14.75 & 0.448 \\
\textit{(c) assets, no action}          & $\times$   & --          & \checkmark & 15.07 & 0.420 \\
\textit{(a$'$) assets $+$ unpaired $z$} & \checkmark & $\times$   & \checkmark & 14.92 & 0.428 \\
\textit{(d) \model~(ours)}              & \checkmark & \checkmark & \checkmark & \textbf{16.35} & \textbf{0.376} \\
\bottomrule
\end{tabular}}
\end{table}

\begin{figure*}[t]
    \centering
    \includegraphics[width=\linewidth]{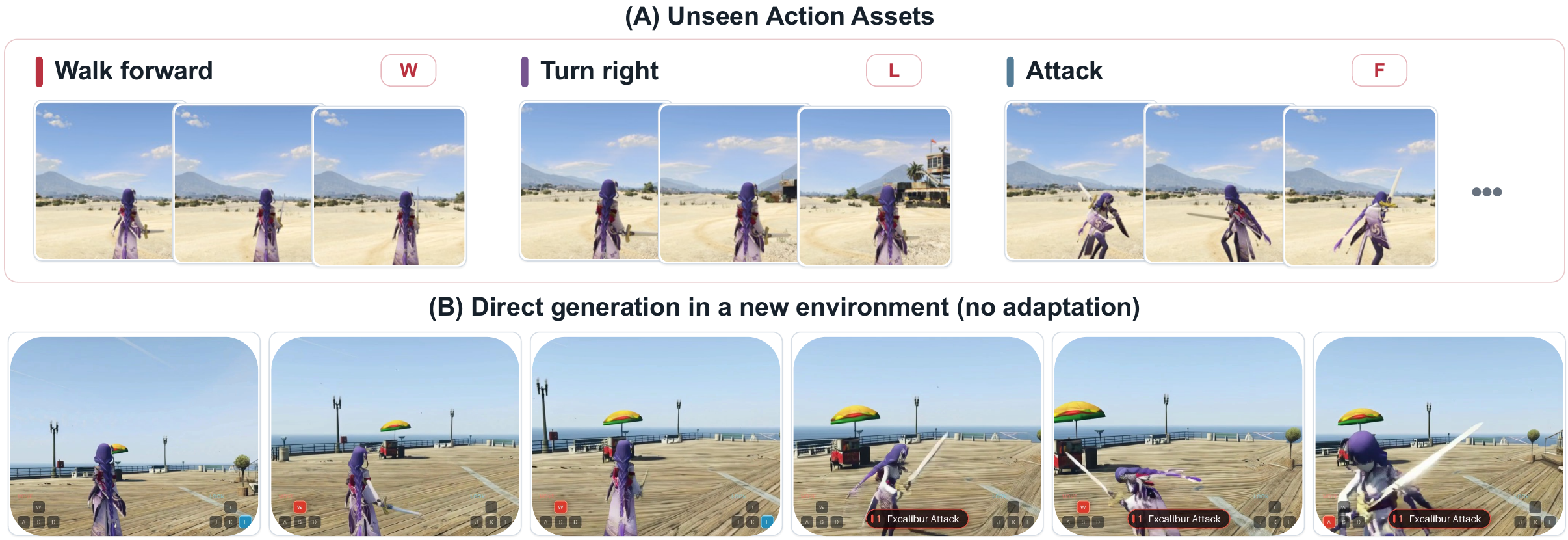}
    \caption{\textbf{Transferring unseen actions:} assets recorded from a modded, unseen character (A) drive generation in a new environment (B).}
    \label{fig:unseen_actions}
\end{figure*}

\subsection{Ablation Studies}
\label{sec:exp_ablation}
\paragraph{Effectiveness of shadow pairing and source assets}
\looseness=-1 As detailed in Table~\ref{tab:ablation}, we study the contribution of each component on the transfer split of Sec.~\ref{sec:exp_transfer}: \textit{(a)} the prior $z$-only latent interface, a separately trained model following the Olaf-World recipe~\cite{jiang2026olaf} on our data; \textit{(b)} the same interface with the LAM trained on shadow pairs instead of self-reconstruction; \textit{(c)} source assets without the action $z$; \textit{(a$'$)} the assets with an \emph{unpaired} $z$ from the Olaf-recipe LAM; and \textit{(d)} the full model; (b), (c), and (d) are inference-time ablations of one checkpoint, while (a) and (a$'$) are separately trained under the matched recipe.
Pairing is the foundation: once the latent is constrained to carry dynamics alone, the same interface becomes transferable ((a) vs.\ (b)).
Assets are a strong carrier on their own, since the source video already contains the action and its spatial detail, so (c) matches or surpasses the bottlenecked $z$.
The decisive comparison is (a$'$) vs.\ (d): they share both a latent and the assets and differ only in whether the latent was trained with shadow pairing, and the paired latent improves PSNR by $1.4$; an unpaired latent, by contrast, adds nothing over assets alone ((a$'$) vs.\ (c)).
The top-line gain is thus attributable to cross-shadow pairing, not to the assets alone or to conditioning on any latent.

\subsection{Transferring Unseen Actions}
\label{sec:exp_unseen}

A latent interface is only as general as its behavior on dynamics it has never seen, so we test transfer on actions introduced through game mods after training: an unseen player character and a two-handed sword whose attack is a swing rather than a shot; neither asset appears in the \shadowlibrary, nor does the motion either one produces.
We record walking, turning, and sword attacks from the modded character as action assets.
The source clips are recorded on one map and the extracted actions replayed on held-out maps, so transfer to a new environment is built into the protocol; Fig.~\ref{fig:unseen_actions} shows the results.
The modded dynamics survive this replay: this character's gait and this sword's swing, motion the model never observed, are reproduced on a held-out map, indicating that the encoder reads dynamics rather than memorized action categories.
This generalization is where the scaling potential of the paradigm lies: a behavior demonstrated after training joins the action library at the cost of one pass through the frozen encoders, so coverage grows with new demonstrations rather than with new training runs.

\section{Conclusion}
\model enables any-action, frame-level control of interactive video world models by demonstration.
Through shadow pairs and cross-shadow prediction, it learns a unified dynamics representation that discards, by construction, whatever the pairing resamples and keeps whatever it preserves, so users can extract the dynamics of any clip as a reusable action asset and replay it in new environments without labels, estimators, or fine-tuning.
Our experiments demonstrate clear improvements over latent-action and interactive world model baselines in both action transfer and long action rollout, across dynamics families ranging from human motion and camera control to gameplay and robot manipulation.
\model opens new possibilities for interactive world models, which can now be driven by showing rather than telling, from entertainment to simulation-based training of embodied agents.
We discuss limitations and future directions in Supplementary Sec.~\ref{app:limitations}.

{
    \small
    \bibliographystyle{ieeenat_fullname}
    \bibliography{main}
}

\clearpage
\setcounter{page}{1}
\maketitlesupplementary

\setcounter{section}{0}
\renewcommand{\thesection}{\Alph{section}}
\renewcommand{\theHsection}{supp.\Alph{section}}
\numberwithin{equation}{section}
\theoremstyle{plain}
\newtheorem{theorem}{Theorem}[section]
\newtheorem{lemma}[theorem]{Lemma}
\newtheorem{corollary}[theorem]{Corollary}
\newtheorem{assumption}{Assumption}
\renewcommand{\theassumption}{A\arabic{assumption}}
\theoremstyle{definition}
\newtheorem{definition}[theorem]{Definition}
\newtheorem{remark}[theorem]{Remark}

\section{Limitations and Future Directions}
\label{app:limitations}

\paragraph{Limitation}
\looseness=-1 \model depends on two preconditions.
First, action assets must be prepared in advance: at deployment the model consumes reference clips or stored assets, so a dynamics for which no demonstration exists must first be captured or authored.
Second, shadow pairs are easy to construct in games and simulators, which replay the same dynamics under resampled appearance on demand, but are much harder to collect in the real world, where such re-runs are rarely possible; real video therefore enters training only as self-pairs, contributing visual realism rather than the identifying signal.

\paragraph{Future directions}
\looseness=-1 Both preconditions point to generative remedies.
First, asset generation: because the encoder is appearance-invariant by construction, a demonstration does not have to be filmed.
Off-the-shelf video generation~\cite{GoogleDeepMindVeo,wan2025wan} is not sufficient by itself, however: a usable asset set consists of many clips of the same subject that stay strictly consistent in appearance and starting state and differ only in the performed action, whereas current models generate each video independently.
Designing an \emph{asset-generation} video model with this cross-clip consistency would let the action library scale with the generative ecosystem rather than with engine coverage.
Second, shadows for real video: since the pairing is a protocol rather than an engine, video editing models that change a clip's appearance while keeping its motion could manufacture shadows of real footage, extending the identifying supervision from synthetic worlds to the real domain.

\section{More Model Details}
\label{app:model_details}

\subsection{Factor-Selective Readout of the Unified Representation}
\label{app:factor_selective}

\looseness=-1 The main text (Sec.~\ref{sec:method_lam}) states that ego and scene dynamics are separated within one latent space; this section gives the mechanism.
The inverse-dynamics encoder follows the standard LAM architecture~\cite{gaoadaworld,jiang2026olaf}: a spatiotemporal transformer with causal temporal attention whose readout at frame $t{+}1$ infers $z_t$, so in implementation the posterior conditions on the full causal prefix, $q_\phi(z_t \mid x_{1:t+1})$, rather than the transition alone; clips of any length are encoded in a single pass.
We use \emph{three prompt slots}: three learnable tokens (\texttt{cam}, \texttt{dyn}, and \texttt{full}) are prepended to every frame's patch sequence, attend jointly through the shared encoder, and each owns a variational readout head.
The pairing protocol routes supervision among them: pairs preserving only the camera trajectory (scene and subject resampled or frozen) supervise the \texttt{cam} head; pairs preserving only scene motion (camera resampled) supervise the \texttt{dyn} head; pairs preserving both supervise the \texttt{full} head.
A per-sample mask $(m_{\mathrm{cam}}, m_{\mathrm{dyn}}) \in \{0,1\}^2$ selects one head, so each head is supervised only by the factor its pairs preserve.
What constitutes the \texttt{dyn} factor is not fixed a priori but defined by each family's pairing protocol: our robot manipulation pairs preserve the arm trajectory alone, so for this family \texttt{dyn} reads out arm dynamics and nothing else in the scene; our human pairs preserve the body motion.
This is the protocol nature of the \shadowlibrary: designating a factor as \texttt{dyn} requires only constructing pairs that preserve it.
The per-source control channel $(m_{\mathrm{cam}}, m_{\mathrm{dyn}})$ is listed in the composition table of Sec.~\ref{app:exp_details}.
All heads share one encoder and one latent space $\mathcal{Z}$, so downstream consumption is identical; yet at inference they are independent dials: the same reference video yields a pure camera action, a pure scene action, or their joint, and a body motion, a camera move, and an object interaction are simply different trajectories $z$.

\subsection{Conditioning Architecture}
\label{app:conditioning}

This section details the conditioning routes summarized in Sec.~\ref{sec:method_worldmodel} of the main text.
The backbone operates on latents of a causal 3D video VAE with temporal compression $r{=}4$; we therefore group the per-frame actions by latent frame, $\mathcal{Z}_f = ( z_{r(f-1)+1}, \dots, z_{rf} )$, following~\cite{yu2025gamefactory}.
The modulation route fuses the group-averaged action $\bar{z}_f$ into the per-latent-frame timestep embedding,
\begin{equation}
    e_f = \psi(\tau) + \alpha_z \, W_{\!z}\, \bar{z}_f,
    \label{eq:adaln}
\end{equation}
where $\psi(\tau)$ embeds the diffusion timestep and $e_f$ drives the adaptive layer norm of every DiT block, providing global, low-frequency control.
The cross-attention route preserves all $r$ action tokens: the spatial hidden states $h_f$ of latent frame $f$ attend to a per-frame context $\mathcal{C}_f$ through a dedicated arm,
\begin{equation}
    h_f \leftarrow h_f + W_{\!o}\,\mathrm{Attn}\big(h_f,\, \mathcal{C}_f\big),
    \label{eq:xattn}
\end{equation}
with $\mathcal{C}_f = [\,W_{\!a}\mathcal{Z}_f\,]$ initially holding the projected action tokens.
$W_{\!o}$ is zero-initialized, so the cross-attention arm is an exact no-op at step~0; the modulation gate $\alpha_z$ is instead initialized to a positive constant, since a zero-initialized gate received vanishing gradient and stalled in practice.

The bottleneck $d_z{=}32$ is a feature rather than a limitation: invariance demands that $z$ be too small to smuggle appearance, but it is then also too small to carry the high-frequency content of a motion, such as the precise articulation of limbs and the exact timing of contacts.
We therefore feed the source video itself to the generator as assets.
Let $s = \mathcal{E}_{\mathrm{vae}}(x^{\mathrm{src}})$ be the source's 3D-VAE latents.
On the input side, $s$ is channel-concatenated with the noisy latents $u$ and a binary availability mask $m$, then embedded by the input 3D convolution:
\begin{equation}
    h^{(0)} = \mathrm{Conv3D}\big(\,[\,u;\, s;\, m\,]\,\big);
    \label{eq:concat}
\end{equation}
on the context side, each latent frame $s_f$ is patchified into tokens that join the cross-attention context of Eq.~\eqref{eq:xattn},
\begin{equation}
    \mathcal{C}_f = \big[\, W_{\!a}\mathcal{Z}_f;\;\, W_{\!s}\,\mathrm{patchify}(s_f) \,\big],
    \label{eq:ctx}
\end{equation}
alongside a per-frame semantic embedding of the source in a third, parallel arm.

\section{Formal Guarantee for Shadow Learning}
\label{app:shadow_guarantee}

\looseness=-1 This section makes precise the sense in which a shadow pair makes an arbitrary
dynamics family learnable.  An unconditional statement is impossible: if two
dynamics produce exactly the same observations, no algorithm can distinguish
them; if a finite-dimensional code is too small for the intrinsic dynamics
dimension, no continuous encoder can store them.  We therefore state the
necessary observability, separation, and capacity conditions explicitly.  The
result is agnostic to the renderer and to whether the dynamics describe a
human, camera, robot, object, or an entire multi-action trajectory.

Our argument is related in spirit to results showing that paired views can
isolate invariant content from independently changing style
~\cite{vonkugelgen2021self} and can restore identifiability in multi-view
nonlinear models~\cite{gresele2020incomplete}.  Here we do not assume
independent latent components or attempt to recover a privileged coordinate
system.  We prove the task-specific statement needed by cross-shadow
prediction: the coarsest statistic sufficient for predicting one shadow from
the other is exactly the shared dynamics, up to an invertible
reparameterization.

\subsection{Setup and Assumptions}

\looseness=-1 Let $D$ denote the dynamics to be represented.  It may be one transition
$d_t$, a variable-length segment $d_{1:T}$, or a full rollout.  Let $X$ be a
source video, let $B$ collect the target-side information already available to
the decoder (for example, the previous target frame or clean block history),
and let $Y$ be the target next frame or block.  All variables take values in
standard Borel spaces, so the conditional distributions below exist.  Define
the target transition kernel
\begin{equation}
    K_d(b) := \mathbb{P}(Y\in\,\cdot\mid B=b,D=d).
    \label{eq:supp_kernel}
\end{equation}
Three explicit conditions capture the shadow construction.

\begin{assumption}[Independent context resampling]
\label{ass:shadow_independence}
The source is conditionally independent of the target rendering once the
shared dynamics is fixed:
\begin{equation}
    X \perp\!\!\!\perp (B,Y)\mid D.
    \label{eq:supp_shadow_ci}
\end{equation}
This is the probabilistic statement of ``replay the dynamics, resample
everything else.''  It allows both views to depend arbitrarily on $D$.
\end{assumption}

\begin{assumption}[Source observability]
\label{ass:shadow_observability}
There is a measurable map $h$ such that
\begin{equation}
    D=h(X) \qquad \text{almost surely}.
    \label{eq:supp_observable}
\end{equation}
Thus the requested dynamics is actually visible in the source clip.  This is
necessary for any deterministic encoder of one video to recover it.
\end{assumption}

\begin{assumption}[Target overlap and separation]
\label{ass:shadow_separation}
There is a probability measure $\nu$ on the target-context space such that
$\mathbb{P}(B\in\,\cdot\mid D=d)$ is equivalent to $\nu$ for almost every $d$;
that is, every dynamics is rendered over the same target contexts up to null
sets.  Moreover, distinct dynamics induce distinct transition kernels:
\begin{equation}
    K_d = K_{d'} \quad\Longrightarrow\quad d=d'.
    \label{eq:supp_separation}
\end{equation}
Here equality means equality $\nu$-almost everywhere.  We also assume that
$d\mapsto[K_d]_\nu$ is a measurable map into a standard Borel coding of these
kernels.  Equivalently, every pair $d\neq d'$ differs in its effect in a set of
target contexts of nonzero probability.  Separation is necessary.  The stated
mutual absolute continuity is a transparent sufficient coverage condition
that prevents the decoder-side context $B$ from perfectly revealing $D$.
\end{assumption}

\begin{definition}[Cross-shadow representation]
\label{def:shadow_sufficiency}
A representation $Z=s(X)$ is \emph{cross-shadow sufficient} when
\begin{equation}
    Y \perp\!\!\!\perp X\mid (B,Z),
    \label{eq:supp_sufficient}
\end{equation}
and is called \emph{minimal} when every other cross-shadow sufficient statistic
determines it.  Minimality does not impose a particular coordinate system; two
minimal statistics may differ by a bijection.
\end{definition}

\subsection{Identification Theorem}

\begin{theorem}[Shadow identification]
\label{thm:shadow_identification}
Under Assumptions~\ref{ass:shadow_independence}--\ref{ass:shadow_separation},
define the conditional-law statistic
\begin{equation}
    \Gamma(x):=[K_{h(x)}]_\nu.
    \label{eq:supp_gamma}
\end{equation}
Then: (i)~$\Gamma(X)$ is cross-shadow sufficient and minimal;
(ii)~it is invariant to every source-context factor resampled by the shadow
construction; and (iii)~every cross-shadow sufficient representation
$Z=s(X)$ determines $D$.  Consequently, every minimal sufficient $Z$ and $D$
determine one another and are equivalent up to a one-to-one
reparameterization.
\end{theorem}

\begin{proof}
Because $D=h(X)$ and $X\perp\!\!\!\perp(B,Y)\mid D$, for almost every $x$ and
target context $b$,
\begin{align}
&\mathbb{P}(Y\in\,\cdot\mid B=b,X=x) \\
&\quad=\mathbb{P}(Y\in\,\cdot\mid B=b,D=h(x)) \\
&\quad=K_{h(x)}(b).
\label{eq:supp_key_identity}
\end{align}
The conditional law of $Y$ given $(B,X)$ therefore depends on $X$ only
through $\Gamma(X)$, which proves sufficiency.  The same identity shows
invariance: source videos with the same $D$ have the same $\Gamma$ regardless
of source context.

It remains to prove minimality without allowing the target-side variable $B$
to do the identification.  Let $Z=s(X)$ be any cross-shadow sufficient
statistic and write
$L_z(b)=\mathbb{P}(Y\in\,\cdot\mid B=b,Z=z)$.  Sufficiency and
Eq.~\eqref{eq:supp_key_identity} give
\begin{equation}
    K_D(B)=L_Z(B) \qquad \text{almost surely}.
    \label{eq:supp_kernel_match}
\end{equation}
Because $Z$ is a function of $X$, Assumption~\ref{ass:shadow_independence} implies
$B\perp\!\!\!\perp Z\mid D$.  Therefore, for almost every pair $(d,z)$ in the
support of $(D,Z)$, Eq.~\eqref{eq:supp_kernel_match} holds under
$\mathbb{P}(B\in\,\cdot\mid D=d)$.  Equivalence of this measure to $\nu$ in
Assumption~\ref{ass:shadow_separation} upgrades this to the kernel-class
identity
\begin{equation}
    [K_D]_\nu=[L_Z]_\nu \qquad \text{almost surely}.
    \label{eq:supp_kernel_class}
\end{equation}
The Borel injection $d\mapsto[K_d]_\nu$ has a measurable inverse on its image,
so
\begin{equation}
    D=(d\mapsto[K_d]_\nu)^{-1}([L_Z]_\nu)=:r(Z)
    \qquad \text{almost surely}.
    \label{eq:supp_recover_d}
\end{equation}
Consequently, $\Gamma(X)=[K_{r(Z)}]_\nu$ is a measurable function of every
sufficient $Z$, proving minimality.  Conversely, the same inverse recovers
$D$ from $\Gamma(X)$.  Any other minimal sufficient $Z$ is also a function of
$\Gamma$, so $Z$ and $D$ are bijective up to null sets.
\end{proof}

\begin{corollary}[Cross-shadow likelihood learns a dynamics representation]
\label{cor:shadow_logloss}
Assume the relevant conditional laws admit densities with respect to a common
dominating measure and have finite population log loss.  Let $Z=s(X)$ and
allow the decoder to range over all conditional
distributions $q(Y\mid B,Z)$.  Under
Assumptions~\ref{ass:shadow_independence}--\ref{ass:shadow_separation}, $Z$ attains the same
population negative log-likelihood as a decoder given the full source $X$ if
and only if $Z$ is cross-shadow sufficient.  Therefore every Bayes-optimal
code contains $D$, and every minimal Bayes-optimal code is a one-to-one
reparameterization of $D$.
\end{corollary}

\begin{proof}
For a fixed $Z$, the optimal decoder is the true conditional law
$\mathbb{P}(Y\mid B,Z)$.  Since $Z$ is a function of $X$, the gap from the
full-source Bayes risk is
\begin{align}
\mathcal R^*(Z)-\mathcal R^*(X)
&=\mathbb E\!\left[
\mathrm{KL}\!\left(
\mathbb P(Y\mid B,X)\,\|\,\mathbb P(Y\mid B,Z)
\right)\right] \\
&\geq 0.
\label{eq:supp_log_gap}
\end{align}
Equality holds exactly when
$Y\perp\!\!\!\perp X\mid(B,Z)$.  The conclusion then follows from
Theorem~\ref{thm:shadow_identification}.
\end{proof}

The theorem identifies a representation, not a preferred axis system: if
$\varphi$ is any bijection, $\varphi(D)$ is equally valid.  This is the
strongest identifiability one can request of an unlabeled latent.  A
nonminimal Bayes-optimal encoder may additionally store source appearance, but
that information is provably unnecessary for cross-shadow prediction and is
removed by passing to the minimal sufficient statistic.  In our model, the
finite variational bottleneck is the mechanism that favors this minimal
solution; the theorem does not claim that a KL penalty or a particular SGD run
uniquely guarantees it.

\begin{remark}[Deterministic representation versus variational channel]
Theorems~\ref{thm:shadow_identification} and
Corollary~\ref{cor:shadow_logloss} concern a deterministic statistic $Z=s(X)$,
such as encoder parameters or a posterior mean.  The sampled Gaussian channel
$q_\phi(z_t\mid x_{1:t+1})$ and the $\beta$-weighted KL term in our implementation form
a rate--distortion approximation to this ideal.  They favor a compact code but
do not, by themselves, guarantee exact sufficiency or minimality at finite
rate.
\end{remark}

\subsection{Realization and Learning with Differentiable Networks}

The previous theorem is nonparametric.  We next show that, under standard
regularity conditions, its representation and predictor can be realized by
differentiable neural networks.

\begin{theorem}[Differentiable neural representation]
\label{thm:shadow_neural_realizability}
Assume Assumptions~\ref{ass:shadow_independence}--\ref{ass:shadow_separation}.
In addition, suppose that the supports of $X$, $B$, and
$D$ are compact subsets of finite-dimensional Euclidean spaces; $h$ is
continuous; $D$ admits a continuous embedding $e:\mathcal{D}\to\mathbb{R}^m$
into the chosen latent dimension; and the target rendering is
\begin{equation}
    Y=F(B,D)
    \label{eq:supp_transition_model}
\end{equation}
for a continuous $F$.  Then there exist sequences of differentiable neural
networks $(E_n,G_n)$ such that
\begin{equation}
\sup_x\|E_n(x)-e(h(x))\|_2\longrightarrow 0
\label{eq:supp_encoder_convergence}
\end{equation}
and
\begin{equation}
\sup_{b,x}\|G_n(b,E_n(x))-F(b,h(x))\|_2\longrightarrow 0.
\label{eq:supp_predictor_convergence}
\end{equation}
Thus a differentiable neural encoder can represent $D$ up to an invertible
coordinate change while discarding source context, and its decoder can attain
arbitrarily small cross-shadow reconstruction error.
\end{theorem}

\begin{proof}
Set $E_0=e\circ h$.  Since $e$ is a continuous injection from a compact space
into a Hausdorff space, its inverse is continuous on $e(\mathcal D)$.  Hence
\begin{equation}
    G_0(b,z)=F\bigl(b,e^{-1}(z)\bigr),
    \qquad z\in e(\mathcal D),
\end{equation}
is continuous and satisfies
$G_0(B,E_0(X))=F(B,D)=Y$.  Extend $G_0$ continuously from
the compact set $\mathcal B\times e(\mathcal D)$ to a compact Euclidean
neighborhood.  Universal approximation with a smooth, bounded, nonconstant
sigmoidal activation (such as logistic or $\tanh$)
then gives sequences of differentiable networks that approximate $E_0$ and
this extension of $G_0$ uniformly~\cite{hornik1991approximation}.
Uniform continuity of the extension of $G_0$ turns both approximations
into Eqs.~\eqref{eq:supp_encoder_convergence} and \eqref{eq:supp_predictor_convergence}.
\end{proof}

\begin{corollary}[Exact reconstruction identifies the representation]
\label{cor:shadow_mse}
In the deterministic setting of
Theorem~\ref{thm:shadow_neural_realizability}, suppose a code $Z=s(X)$ and
decoder $G$ attain
\begin{equation}
    \mathbb E\|Y-G(B,Z)\|_2^2=0.
    \label{eq:supp_zero_mse}
\end{equation}
Then $Z$ determines $D$ almost surely.  If $Z$ is minimal among exact
reconstruction codes, it is a one-to-one reparameterization of $D$.
\end{corollary}

\begin{proof}
Equation~\eqref{eq:supp_zero_mse} implies $Y=G(B,Z)$ almost surely, hence
$Y\perp\!\!\!\perp X\mid(B,Z)$.  The claim follows from
Theorem~\ref{thm:shadow_identification}.  Conversely, in this deterministic
setting any sufficient $Z$ admits the exact decoder
$G(B,Z)=\mathbb E[Y\mid B,Z]=Y$, so minimal exact-reconstruction codes and
minimal sufficient codes coincide.
\end{proof}

Theorem~\ref{thm:shadow_neural_realizability} is a realizability statement.
For stochastic renderers, the same construction applies if the rendering
randomness is included in $B$ and available to the decoder; alternatively one
may approximate the full
conditional kernel with a neural density decoder and use
Corollary~\ref{cor:shadow_logloss}.  There is also a standard statistical
learning counterpart: for i.i.d. shadow pairs, bounded second moments,
globally optimized empirical squared loss, and a neural-network sieve whose
capacity grows at a controlled rate, neural nonparametric regression is
risk-consistent~\cite{white1990connectionist}.  This consistency result learns
the composite Bayes predictor; it does not by itself imply convergence of the
internal encoder.  The representation conclusion follows separately, at a
minimal population optimum, from
Theorem~\ref{thm:shadow_identification} and
Corollaries~\ref{cor:shadow_logloss} and~\ref{cor:shadow_mse}.  It does not assert finite-sample recovery,
a particular convergence rate, convergence of every encoder parameterization,
or that arbitrary nonconvex optimization finds a global solution.

\subsection{Scope of the ``Any Dynamics'' Claim}

The identification result places no semantic restriction on $D$ and no
component-independence assumption on the renderer.  For heterogeneous action
families, the guarantee applies family by family: let $F$ index the pairing
protocol (the family), and let Assumptions~\ref{ass:shadow_independence}--%
\ref{ass:shadow_separation} hold conditionally on $F$ with a family-specific
context measure $\nu_F$.  The overlap condition of
Assumption~\ref{ass:shadow_separation} is then required only \emph{within} a
family, never across families; a human context need never support a robot
dynamics.  Since $F$ is itself observable from the source (which family a clip
shows is visible in the clip), conditioning on $F$ loses nothing, and the
theorem identifies $D$ within every family.  What is shared \emph{across}
families---one encoder and one latent space $\mathcal{Z}$---is an
architectural choice rather than a consequence of the theorem: it is what
makes the per-family identified representations unified in practice, as the
cross-family transfer results of the main text support empirically.
Therefore, any video dynamics for which one can (i) replay the
dynamics while independently resampling the other generative factors,
(ii) observe the requested dynamics in the source, and (iii) render
overlapping contexts, within its family, in which distinct dynamics have
distinct effects admits a context-invariant cross-shadow representation.  Under the additional
finite-dimensionality, compactness, continuity, and latent-capacity conditions
of Theorem~\ref{thm:shadow_neural_realizability}, that representation and its
decoder admit differentiable neural realizations to arbitrary accuracy.

If separation fails, define
\begin{equation}
    d\sim d' \quad\Longleftrightarrow\quad K_d=K_{d'}.
\end{equation}
Provided this equivalence class is itself observable from $X$, the same proof
identifies $[D]$---exactly the part of the dynamics that can affect a shadow---
and no predictive method can do better.  If source observability fails even
for $[D]$, a deterministic representation of one source video is impossible;
in general only the conditional-law statistic
$x\mapsto\mathbb P(Y\in\,\cdot\mid B=\cdot,X=x)$ is identified (it may be
representable as a mixture over equivalence classes).  These
boundaries prevent ``any dynamics'' from being read as a promise to recover
information that no video contains.

Finally, the guarantee applies to exact shadows satisfying
Assumption~\ref{ass:shadow_independence}.  The
degenerate self-pairs used to import unpaired real video are an empirical
addition and do not, by themselves, provide the identifiability guarantee.

\section{Additional Experimental Details and Ablations}
\label{app:exp_details}

\subsection{Composition of the \shadowlibrary}
Table~\ref{tab:data} summarizes the sources of the \shadowlibrary: how each source is paired, which control channel it supervises via the per-sample mask $(m_{\mathrm{cam}}, m_{\mathrm{dyn}})$ of Sec.~\ref{app:factor_selective}, and its approximate scale and share of the sampling mixture.

\begin{table*}[t]
\centering
\caption{\textbf{Composition of the \shadowlibrary.} Shadow pairs are two frame-synchronized renders of one dynamics trajectory with the remaining generative factors resampled; unpaired corpora enter as self-pairs. $(\text{cam},\text{dyn})$ is the control channel each source supervises. Counts are approximate.}
\label{tab:data}
\resizebox{\linewidth}{!}{
\begin{tabular}{llccccc}
\toprule
Category & Data source & Pairing & (cam,\,dyn) & \#Seqs & \#Frames & Mixture \\
\midrule
\multirow{4}{*}{Human motion}
 & SMPL-X re-render, body dynamics only         & shadow pair & (0,\,1) & ${\sim}1$k   & ${\sim}1.1$M & 13.6\% \\
 & SMPL-X re-render, camera only                & shadow pair & (1,\,0) & ${\sim}1.1$k & ${\sim}1.2$M & 10.9\% \\
 & SMPL-X re-render, camera $+$ body            & shadow pair & (1,\,1) & ${\sim}1$k   & ${\sim}1.2$M & 8.2\% \\
 & Paired character renders~\cite{cao2025uni3c} & shadow pair & (1,\,1) & ${\sim}2$k   & ${\sim}0.3$M & 8.8\% \\
\midrule
\multirow{3}{*}{Robot manipulation}
 & ManiSkill~\cite{tao2024maniskill3}, arm only       & shadow pair & (0,\,1) & ${\sim}1.1$k & ${\sim}1$M   & 2.7\% \\
 & ManiSkill~\cite{tao2024maniskill3}, camera only    & shadow pair & (1,\,0) & ${\sim}1.2$k & ${\sim}1.1$M & 6.8\% \\
 & ManiSkill~\cite{tao2024maniskill3}, camera $+$ arm & shadow pair & (1,\,1) & ${\sim}1.2$k & ${\sim}1.1$M & 5.4\% \\
\midrule
\multirow{2}{*}{First-person games}
 & GTA-style urban driving and weapon fire & shadow pair & (1,\,1) & ${\sim}0.9$k & ${\sim}0.4$M & 8.8\% \\
 & Cyberpunk-style night city              & shadow pair & (1,\,1) & ${\sim}3.4$k & ${\sim}1.6$M & 8.2\% \\
\midrule
\multirow{2}{*}{Third-person games}
 & Unreal Engine scenes (melee, spellcasting) & shadow pair & (1,\,1) & ${\sim}30$ scenes & ${\sim}0.6$M & -- \\
 & Monster Hunter-style action                & shadow pair & (1,\,1) & ${\sim}1$k   & ${\sim}0.5$M & -- \\
\midrule
Static-scene camera
 & DL3DV~\cite{ling2023dl3dv} & self-pair & (1,\,0) & ${\sim}1$k scenes & ${\sim}0.3$M & 6.8\% \\
\midrule
\multirow{7}{*}{Unpaired real video}
 & MiraData Internet clips~\cite{ju2024miradata} & self-pair & (1,\,1) & ${\sim}5$k   & ${\sim}6.5$M & 5.4\% \\
 & OpenX BridgeData~\cite{oneill2024openx}       & self-pair & (0,\,1) & ${\sim}29$k  & ${\sim}0.8$M & 0.1\% \\
 & OpenX FurnitureBench~\cite{oneill2024openx}   & self-pair & (0,\,1) & ${\sim}5$k   & ${\sim}2.2$M & 4.1\% \\
 & OpenX Berkeley UR5~\cite{oneill2024openx}     & self-pair & (0,\,1) & ${\sim}1$k   & ${\sim}0.1$M & 2.7\% \\
 & OpenX Jaco Play~\cite{oneill2024openx}        & self-pair & (0,\,1) & ${\sim}1.1$k & ${\sim}0.1$M & 2.7\% \\
 & OpenX Stanford HYDRA~\cite{oneill2024openx}   & self-pair & (0,\,1) & ${\sim}0.6$k & ${\sim}0.3$M & 2.7\% \\
 & OpenX RoboTurk~\cite{oneill2024openx}         & self-pair & (0,\,1) & ${\sim}2$k   & ${\sim}0.1$M & 2.0\% \\
\bottomrule
\end{tabular}}
\end{table*}

\subsection{Training and Inference Details}
\looseness=-1 The LAM is trained with $\beta{=}0.01$ and a fixed prior $\mathcal{N}(0, I)$, always at half the world model's spatial resolution; real videos enter the stream as degenerate self-pairs.
The self-pair probability is $0.5$ for the human body-dynamics source and $0$ for the other paired sources; since the unpaired real-video and static-scene camera sources are always self-paired, roughly one third of training samples are self-pairs.
The world model is first fine-tuned bidirectionally with flow matching, then converted to a block-causal generator with 3 latent frames per block and rolled out with a key--value cache at inference.
The action-transfer comparisons (Table~\ref{tab:transfer}) use the model at $480{\times}720$ (LAM at $240{\times}360$); the long-rollout model (Table~\ref{tab:rollout}) generates at $544{\times}960$ (LAM at $272{\times}480$), warm-started from a lower-resolution checkpoint and fine-tuned at the target resolution for 10k steps on $8{\times}$H200 GPUs with fully-sharded data parallelism.
At inference we sample with 30 flow-matching denoising steps and classifier-free guidance 5.0.

\subsection{Evaluation Details}
In Table~\ref{tab:transfer}, the human-motion, first-person, and third-person families are scored on 81-frame clips from the first frame, averaging 45--50 held-out pairs each; robot manipulation uses a smaller set of held-out ManiSkill pairs, and the camera split uses held-out camera-trajectory pairs.
Since monocular pose recovery is scale-ambiguous, the VGGT-recovered trajectory is aligned to the target with a similarity transform (Sim(3)) before computing ATE/RPE.
The ablation of Table~\ref{tab:ablation} uses a 12-pair subset of the transfer split under the same protocol, covering the first-person combat and third-person action families; the reported numbers average the two.
\paragraph{Long-rollout judging protocol}
The comparison of Table~\ref{tab:rollout} uses 16 long command streams.
For each stream, \model~and the baseline generate from the same first frame under the same commands; both rollouts are cut at the same boundaries into 240-frame segments, and corresponding segments are compared, giving roughly 64 comparison pairs per baseline.
Each pair is anonymized as A and B with randomized assignment and judged by a VLM (Fable~5) in a forced choice on each of the three axes, with no tie option; a human audit reviews a random $20\%$ of the judgments to validate the VLM's decisions.
The judge receives both segments, the command list, and an instruction of the following form:
\begin{quote}\small
You are judging two video rollouts, A and B, generated by two different systems from the same first frame under the same sequence of action commands: \texttt{<commands>}.
Judge the actions, not the image quality.
Answer three independent questions; for each you must answer exactly ``A'' or ``B'', even if the difference is small.
(1)~\emph{Action control}: in which video are the commanded actions actually performed, at the commanded times?
(2)~\emph{Action fidelity}: which video better preserves the specific weapon, object, and motion of each commanded action?
(3)~\emph{Long-horizon consistency}: which video stays coherent for longer, without drifting away from the commanded behavior, freezing, or degrading as the rollout proceeds?
Ignore differences in visual style, sharpness, or aesthetics, except where degradation makes the action itself unreadable.
Output exactly three lines: \texttt{control: A|B}, \texttt{fidelity: A|B}, \texttt{consistency: A|B}.
\end{quote}

\subsection{Representation-Level Probes}
Before touching the world model, we probe the latent $z$ itself on held-out shadow pairs (Table~\ref{tab:ablation_probe}).
First, the \emph{cross-pair transfer ratio}: with the LAM's own decoder, we reconstruct the target from its visual context using either the target's own action ($z(\tilde{x})$, ``self'') or the shadow source's ($z(x)$, ``cross''); the ratio $\mathrm{MSE}_{\text{cross}}/\mathrm{MSE}_{\text{self}}$ measures how much a cross-appearance $z$ degrades reconstruction.
Second, two linear probes on pooled $z$ over held-out UE clips: decoding the \emph{character} and the \emph{scene}.
In this evaluation set each character is given its own distinct action set, so character identity is readable from motion alone and the character probe measures action content (higher is better), while the scene probe measures pure appearance (lower is better).
Since characters differ in appearance as well as in motion, character accuracy alone could in principle reflect either; the baseline contrast below rules out the appearance route: the self-reconstruction latent leaks \emph{more} appearance by the scene probe yet decodes character only at chance, so character accuracy tracks the motion set rather than looks.
The cross-shadow latent transfers almost losslessly on first-person pairs (ratio $1.02$ vs.\ $1.27$) and dominates both probes: it decodes character at $11\times$ chance while leaking less scene identity, whereas the self-reconstruction latent is at chance on content yet leaks \emph{more} appearance --- regularization does not substitute for pairing.
Note that the cross/self ratio is informative only for a content-bearing latent: the baseline decodes content at chance, so its near-unity third-person ratio reflects an uninformative $z$, for which self and cross reconstructions are equally unguided, rather than successful transfer.

\begin{table}[t]
\centering

\caption{\textbf{LAM representation probes} on held-out shadow pairs (24 pairs) and UE clips (60 clips; character chance $=0.04$, scene chance $=0.20$). The Olaf-recipe row follows~\cite{jiang2026olaf}: self-reconstruction with feature alignment and a single head.}
\label{tab:ablation_probe}
\resizebox{\linewidth}{!}{
\begin{tabular}{l|cc|cc}
\toprule
& \multicolumn{2}{c|}{Cross/self MSE ratio$\downarrow$} & \multicolumn{2}{c}{Linear probe on $z$} \\
LAM training & 1st-person & 3rd-person & Char.\ (content)$\uparrow$ & Scene (appear.)$\downarrow$ \\
\midrule
Self-recon.+reg., single head & 1.265 & 1.067 & 0.050 & 0.650 \\
Cross-shadow, 3-prompt (ours) & \textbf{1.017} & 1.103 & \textbf{0.450} & \textbf{0.433} \\
\bottomrule
\end{tabular}}
\end{table}

\subsection{Distribution-Level Quality}
Reconstruction metrics could in principle be won by a conservative model that stays close to the reference at the cost of realism.
We rule this out with Fr\'echet Video Distance (FVD) on the same generations as Table~\ref{tab:transfer}, computed over 16-frame windows (stride 8) on the three large reconstruction families (Table~\ref{tab:fvd}); the camera family is scored by trajectory error rather than reconstruction, and the robot split is too small for a stable estimate.
\model~roughly halves the FVD of Olaf-World on every family, so its reconstruction advantage comes with better, not worse, distribution-level realism.

\begin{table}[t]
\centering
\caption{\textbf{Distribution-level quality (FVD$\downarrow$)} on the same generations as Table~\ref{tab:transfer}.}
\label{tab:fvd}
\begin{tabular}{lcc}
\toprule
Family & Olaf-World & \model \\
\midrule
First-person combat & 555.8 & \textbf{318.3} \\
Third-person action & 842.0 & \textbf{422.2} \\
Human motion        & 511.6 & \textbf{253.9} \\
\bottomrule
\end{tabular}
\end{table}

\subsection{What the Action Latent Carries}
The conditioning consumes both the latent $z$ and the source assets $s$; this probe tests, at inference, which one carries the action.
For each family we take two shadow pairs $A$ and $B$ from different sources and, from the base input $(z_A, s_A)$, form four interventions: replacing the latent ($z_B, s_A$), shuffling the latent along time, replacing the assets ($z_A, s_B$), and degrading the assets to a quarter resolution ($z_A, s_A^{\downarrow4}$).
We report PSNR against target $A$ (Table~\ref{tab:mismatch}).
On four of five families, replacing or shuffling $z$ drops PSNR sharply while degrading $s$ barely moves it: the commanded action lives in the latent, and the assets supply appearance and high-frequency detail.\footnote{On the body-motion family the source render already depicts the motion in full, so the assets alone can clone it and the probe cannot separate the two channels; the character-render source shares this property and is likewise not probed. This is a limit of the probe on exact-replay families, consistent with the assets being a strong carrier ((c) in Table~\ref{tab:ablation}).}

\begin{table}[t]
\centering
\caption{\textbf{What the action latent carries.} PSNR against target $A$ under each intervention, 3--4 couples per family; $\Delta$ relative to \emph{base}.}
\label{tab:mismatch}
\resizebox{\linewidth}{!}{
\begin{tabular}{lccccc}
\toprule
Family & base & $z{\leftarrow}B$ & shuffle $z$ & $s{\leftarrow}B$ & degrade $s$ \\
\midrule
First-person combat   & 14.77 & $-2.59$ & $-1.57$ & $-2.48$ & $-0.36$ \\
Third-person action   & 16.55 & $-3.01$ & $-0.33$ & $-2.15$ & $-0.41$ \\
Camera                & 21.73 & $-1.92$ & $-2.09$ & $-3.28$ & $-0.13$ \\
Robot arm             & 23.71 & $-0.98$ & $-1.95$ & $-0.15$ & $+0.20$ \\
Body motion           & 22.25 & $-0.53$ & $+0.19$ & $-0.79$ & $-0.37$ \\
\bottomrule
\end{tabular}}
\end{table}

\subsection{Architecture Ablations}
\label{sec:suppl_ablation}

\paragraph{Latent size and self-pair ratio}
The choice $d_z{=}32$ is inherited from prior latent-action models~\cite{gaoadaworld,jiang2026olaf}, where this size was validated as the best operating point for self-reconstruction; since real video enters our training stream as self-pairs under the same reconstruction objective (Sec.~\ref{sec:method_library}), their operating point carries over, and it is consistent with the invariance argument of Sec.~\ref{app:conditioning} that the bottleneck should stay too small to smuggle appearance.
A budget-matched sweep confirmed this choice: the representation probes were insensitive to $d_z\in\{16,64\}$ and to forcing the self-pair ratio, while reconstruction favored $d_z\le 32$, so we retain $d_z{=}32$.

\end{document}